\documentclass[letterpaper, 10 pt, conference]{ieeeconf}  
\IEEEoverridecommandlockouts                              
\usepackage{amsmath,amssymb,amsfonts,mathtools}
\usepackage{amsthm}
\usepackage{booktabs}
\let\labelindent\relax                 
\usepackage{comment}
\usepackage{enumitem}
\usepackage[table]{xcolor}
\usepackage{hyperref}
\usepackage{microtype}
\usepackage{newtxtext,newtxmath} 
\usepackage{algorithm}
\usepackage{algpseudocode}
\usepackage{multirow}
\usepackage{cleveref}
\usepackage{pgfplots}
\pgfplotsset{compat=1.18}
\usepackage{tikz}
\usetikzlibrary{positioning,arrows.meta,calc,shadows}
\pdfmapfile{=pdftex.map}

\usepackage{import}
\usepackage{xifthen}
\usepackage{pdfpages}
\usepackage{transparent}

\hypersetup{
  colorlinks=true,
  linkcolor=blue!60!black,
  citecolor=blue!60!black,
  urlcolor=blue!60!black
}

\newtheorem{theorem}{Theorem}
\newtheorem{proposition}{Proposition}
\newtheorem{lemma}{Lemma}

\theoremstyle{definition}
\newtheorem{definition}{Definition}
\theoremstyle{plain}
\newtheorem{example}{Example}
\newtheorem{problem}{Problem}
\newtheorem{remark}{Remark}

\renewcommand{\paragraph}[1]{\vspace{0.5ex}\noindent\textbf{#1}}

\newcommand{\R}{\mathbb{R}}

\DeclareMathOperator{\enc}{enc} 
\DeclareMathOperator{\dec}{dec} 
\newcommand{\Prob}{\mathbb{P}}
\newcommand{\PredLib}{\mathcal{P}}
\newcommand{\SpecFam}{\mathcal{F}}
\newcommand{\Iset}{\mathcal{I}}

\newcommand{\basis}{{\mathcal{B}}}
\newcommand{\Aset}{\mathcal{A}}
\newcommand{\Kmax}{K_{\max}}
\newcommand{\hor}{\mathrm{hor}}
\newcommand{\N}{\mathbb{N}} 
\newcommand{\supp}{\mathrm{supp}}

\newcommand{\suppU}{\mathrm{supp}^{\PredLib}}
\newcommand{\safe}{\ensuremath{\textsc{Safe}}}

\newcommand{\uncertain}{\ensuremath{\textsc{Uncertain}}}
\newcommand{\pAlways}{\boxdot}
\newcommand{\pEventually}{\Diamonddot}
\newcommand{\qhat}{\widehat{Q}}
\newcommand{\x}{{x}} 
\newcommand{\X}{\mathbb{X}} 
\renewcommand{\O}{\mathbb{O}} 
\newcommand{\Z}{\mathbb{Z}} 

\definecolor{grayfilling}{gray}{0.95} 
\definecolor{grayshadow}{gray}{0.5} 
\definecolor{mydarkblue}{HTML}{003a7d}
\definecolor{mypurple}{HTML}{c701ff}
\definecolor{mymidblue}{HTML}{008dff}
\colorlet{myblue}{mymidblue}
\definecolor{mygreen}{HTML}{4ecb8d}
\definecolor{myorange}{HTML}{ff9d3a}
\definecolor{myyellow}{HTML}{f9e858}
\definecolor{myred}{HTML}{d83034}
\colorlet{accent}{mypurple}

\newcommand{\Oliver}[1]{{\color{red}[Oliver: #1]}}

\title{\LARGE\bf 
Vision-Based Runtime Monitoring under Varying Specifications using Semantic Latent Representations}
\author{
Bardh Hoxha$^{1}$, Oliver Schön$^{2}$, Hideki Okamoto$^{1}$, Lars Lindemann$^{2}$, Georgios Fainekos$^{1}$%
\thanks{$^{1}$B. Hoxha, H. Okamoto, and G. Fainekos are with Toyota NA R\&D}%
\thanks{$^{2}$Oliver Schön and Lars Lindemann are with ETH Z\"urich}
}

\begin{document}
\maketitle
\thispagestyle{empty}
\pagestyle{empty}

\begin{abstract}
We study certified runtime monitoring of past-time signal temporal logic (ptSTL) from visual observations under partial observability. The monitor must infer safety-relevant quantities from images and provide finite-sample guarantees, while being \emph{reusable}: once trained and calibrated, it should certify any formula in a target fragment without per-formula retraining. For fragments induced by a finite dictionary of temporal atoms, we prove that the \emph{semantic basis}, the vector of atom robustness scores, is the minimum prediction target within the class of monotone, 1-Lipschitz reusable interfaces: any formula is evaluated by a deterministic decoder derived from the parse tree, and a single conformal calibration pass certifies the entire fragment with no union bound. We also introduce a \emph{rolling prediction monitor} that predicts only current predicate values and reconstructs temporal history online; this is easier to learn but grows conservative at long horizons. On a pedestrian-crossroad benchmark, rolling achieves tighter certified bounds at short horizons while the semantic-basis monitor is up to 4-times tighter at long horizons. We validate the presented monitors on real-world Waymo driving data, where both monitors satisfy the conformal coverage guarantee empirically.
\end{abstract}

\section{INTRODUCTION}\label{sec:intro}

Runtime monitors provide a mechanism for assessing whether specified safety conditions are satisfied during deployment of autonomous systems. In practice, these conditions are often not fixed once and for all: different missions may require different safety and performance specifications, and operators may update the specifications used at deployment. Accordingly, a practical runtime monitor should be \emph{reusable} (see Fig.~\ref{fig:overview}): after training and calibration, it should support certification for a range of specifications in a target fragment without requiring retraining for each new specification. This motivates a \emph{semantic interface} between observations and specifications: a pre-trained encoder produces a fixed intermediate representation, and formula-specific values are then computed at query time by a deterministic, analytically derived decoder, without additional learning.



Partial observability introduces an additional challenge. The monitor has access only to visual observations, while safety predicates are defined over latent physical quantities (e.g., distances, velocities, and clearances) but the monitor observes only images. Safety-relevant quantities must therefore be inferred from pixels, and the resulting prediction uncertainty must be accounted for in the certificate. This paper combines the problems of reusable specification monitoring and certified learning under partial observability.

A formula-specific certified baseline is to predict the satisfaction measure of a fixed formula and apply conformal prediction to obtain a certified lower bound~\cite{lindemann2023conformal}. This works, but offers no reuse: when the specification changes, both the predictor and the calibration are tied to that formula. To avoid this limitation, the monitor must predict a \textit{reusable} intermediate representation. The choice of representation determines the scope of reuse, the difficulty of visual prediction, and the tightness of the resulting conformal bounds.

\begin{figure*}[t]
    \centering
    \def\svgwidth{.85\linewidth}
    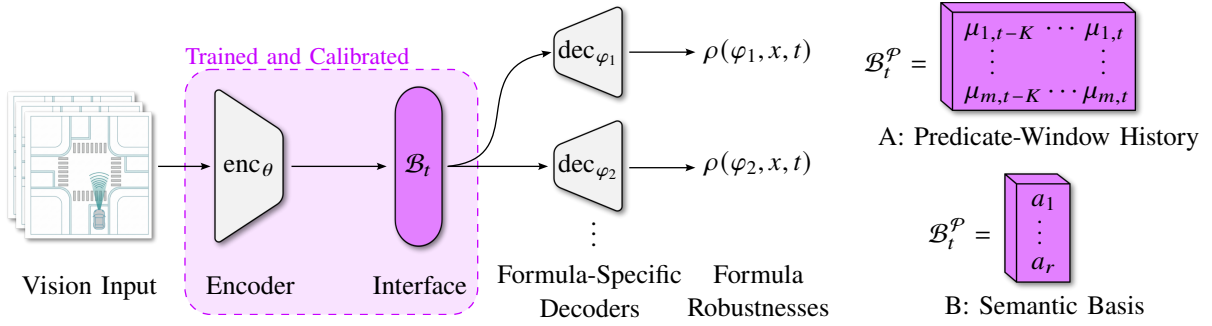

    \caption{\emph{Left:} A single trained encoder $\enc_\theta$ maps vision inputs to a latent interface $\basis_t$ from which formula-specific decoders evaluate any formula $\varphi\in\SpecFam$ in a target fragment $\SpecFam$. \emph{Right:} Two choices of interface basis.}
    \label{fig:overview}
    \vspace{-10pt}
\end{figure*}



We investigate two reusable monitoring interfaces. The first predicts all safety-relevant quantities at each timestep within the specification’s look-back window. This is maximally flexible, since any temporal property can be evaluated from it, but the encoder must regress a high-dimensional output whose size grows with window length. For a specification fragment over a fixed set of temporal operators such as “always safe over the last $K$ steps” or “eventually reach the goal within $K$ steps”, we prove that a strictly smaller representation, the \emph{semantic basis}, suffices to evaluate every specification in the family, and that no smaller representation can.


We use \emph{conformal prediction} (CP)~\cite{lindemann2023conformal,cairoli2023conformal,zhao2024robust} to convert prediction residuals into certified lower bounds, and show that the tightness of these bounds depends critically on whether calibration is applied before or after temporal aggregation in the decoder.

\medskip

\noindent\textbf{Contributions:}
\begin{enumerate}[leftmargin=1.5em]
    \item \emph{Semantic basis as a reusable interface (Section~\ref{sec:exact}):}
    For any ptSTL fragment induced by a finite atomic dictionary, we prove that the semantic basis is minimal within the class of monotone, $1$-Lipschitz reusable interfaces.
    

    \item \emph{Compositional conformal certification (Section~\ref{sec:conformal}):}
    Because every formula is decoded by a monotone, 1-Lipschitz function, single atom-wise conformal bounds certify the entire fragment simultaneously (Thm.~\ref{thm:validity}).

    \item \emph{Rolling prediction monitor:}
    We introduce a rolling prediction monitor that updates the predicate basis online. This reduces the encoder dimension substantially, making the learning problem substantially easier. 
    We provide empirical evidence that this results in higher prediction accuracy and tighter conformal bounds at short horizons.
\end{enumerate}

\medskip


\section{RELATED WORK}
\label{sec:related}
Signal temporal logic (STL) provides a framework for specifying and evaluating temporal properties of continuous-valued signals, with robustness semantics quantifying distance from violation~\cite{fainekos2009robustness,deshmukh2017robust}. Efficient online monitoring algorithms for past-time STL are well established~\cite{dokhanchi2014line,yamaguchi2024rtamt}. We build on these semantics while addressing partial observability through learned vision models.

Conformal prediction (CP) has been used in STL monitoring to provide finite-sample coverage guarantees~\cite{lindemann2023conformal,cairoli2023conformal,zhao2024robust}.  Beyond monitoring, CP has also been used in safe planning and prediction~\cite{lindemann2023safe,dixit2023adaptive}; unlike these works, we study reusable fragment-wide certification from a single calibration pass.

Neural monitoring under partial observability has been studied for fixed properties and small template families~\cite{bortolussi2019neural,cairoli2021neural}. We extend this setting to certify an entire $\wedge/\vee$-closed fragment from a single predictor and characterize the minimal interface required to do so.

Compositional uncertainty-aware STL semantics have been studied under partial and uncertain observations. Robust satisfaction intervals for partial traces were introduced in~\cite{deshmukh2017robust}; interval-valued semantics that propagate uncertainty through STL operators were developed in~\cite{zhong2021extending,baird2023interval}; and a related setting based on affine arithmetic and SMT is studied in~\cite{finkbeiner2022truly}. Our rolling and semantic-basis monitors build on this compositional viewpoint, adding a minimality result for the prediction target and an explicit calibration tradeoff analysis. In ~\cite{balakrishnan2021percemon,hekmatnejad2024formalizing}, the papers present monitoring for perception systems, but do not provide finite-sample certified ptSTL robustness bounds from learned latent visual representations.

\emph{Predictive state representations} (PSRs) construct minimal sufficient statistics for partially observable systems~\cite{littman2001predictive,singh2004PSR}: a linear PSR identifies the minimum-rank basis from which any observable test can be linearly decoded, without modeling a latent belief state explicitly. More recently, temporal logic specifications have been embedded directly into latent spaces via \emph{embedding temporal logic} (ETL), with satisfaction checked through learned distance thresholds, but without certified coverage bounds~\cite{kapoor2025pretrained}. Concept embedding models extend latent representations to probabilistic concept membership for concurrent concept reasoning, also without formal guarantees~\cite{desantis2025v}. Our semantic basis is the ptSTL-monitoring analogue of a PSR: the minimum statistic for which monotone, 1-Lipschitz decoders suffice over all signals, and it is precisely this Lipschitz restriction that enables the conformal certification in Section~\ref{sec:conformal}.

Adaptive conformal methods offer orthogonal improvements to trajectory-level tightness under distribution shift; integrating them with our compositional certification structure is a possible direction for future work~\cite{gibbs2021adaptive}.

\section{PROBLEM FORMULATION}
\label{sec:problem}

\subsection{Dynamical System and Observations}\label{sec:system}
We consider discrete-time dynamical systems with a state $x_t\in\X\subset\R^{d_x}$ evolving as
\begin{align}
    x_{t+1} &= f_X(x_t,u_t, v_t), \qquad
    o_{t} = f_O(x_t, w_t),\label{eq:system}
\end{align}
where $u_t\in\R^{d_u}$ is a control input and $v_t, w_t$ are noise terms. At each time $t\in\Z_{\ge 0}$, the state generates an observation $o_t\in\O\subset\R^{d_o}$; e.g., $o_t$ may be overhead camera images subject to different sensor nuisance. In this paper, the monitor has access only to the observation sequence $\{o_\tau\}_{\tau\le t}$; the state $x_t$ is never directly observed.

\begin{remark}
While we use the state-space model \eqref{eq:system} to fix notation, the monitoring framework requires only a discrete-time signal $x_{0:T}$ paired with an observation sequence $o_{0:T}$. No specific structure on $f_\X$ or $f_\O$ is assumed beyond the exchangeability of individual episodes (Section~\ref{sec:conformal}).
\end{remark}

\subsection{Temporal Logic Specifications}
\label{sec:ptSTL}

Let $\PredLib=\{\mu_1,\dots,\mu_m\}$ be a finite set of \emph{atomic predicates} over signals $\x\colon\Z_{\ge 0}\to\X$, where each $\mu_k\colon\X\times\Z_{\ge 0}\to\{\top,\bot\}$ is defined by a scalar predicate function $h_k\colon\X\to\R$ via 
$$\big(\mu_k(x,t)=\top \big) \Leftrightarrow h_k(x_t)\ge 0.$$ 
We write $\rho(\mu_k,\x,t):=h_k(x_t)$ for the \emph{robustness} of $\mu_k$.

\begin{definition}[ptSTL syntax and quantitative semantics]\label{def:ptSTL}
    \emph{Past-time STL} (ptSTL)~\cite{maler2004monitoring} formulas over predicates $\PredLib$ in \emph{positive normal form} (PNF) are given by the grammar
    \begin{align*}\label{eq:ptSTL_syntax}
        \varphi ::= \mu_k \mid \varphi_1\wedge\varphi_2 \mid \varphi_1\vee\varphi_2 \mid \pAlways_{[a,b]}\,\varphi \mid \pEventually_{[a,b]}\,\varphi,
    \end{align*}
    where $\mu_k\in\PredLib$, $\varphi_1,\varphi_2$ are ptSTL formulas, and $[a,b]\subseteq\Z_{\ge 0}$. The \emph{robustness} $\rho(\varphi,\x,t)\in\R$ is defined recursively over a signal $\x\colon\Z_{\ge 0}\to\X$ at time $t\in \Z_{\ge 0}$ as:
    \begin{align*}
        \rho(\mu_k, x, t) &= h_k(x_t),\\
        \rho(\varphi_1\wedge\varphi_2,\x,t) &= \min\big\lbrace\rho(\varphi_1,\x,t),\,\rho(\varphi_2,\x,t)\big\rbrace,\\
        \rho(\varphi_1\vee\varphi_2,\x,t) &= \max\big\lbrace\rho(\varphi_1,\x,t),\,\rho(\varphi_2,\x,t)\big\rbrace,\\
        \rho(\pAlways_{[a,b]}\varphi,\x,t) &= \inf_{t'\in[t-b,t-a]}\rho(\varphi,\x,t'),\\
        \rho(\pEventually_{[a,b]}\varphi,\x,t) &= \sup_{t'\in[t-b,t-a]}\rho(\varphi,\x,t').
    \end{align*}
    A signal $\x$ satisfies a formula $\varphi$ at time $t$ iff $\rho(\varphi,\x,t)\ge 0$.
\end{definition}

We define the \emph{formula horizon} $\hor(\varphi)\in\N$ as the largest backward time lag needed to evaluate $\varphi$ at time $t$, so that $\rho(\varphi,\x,t)$ depends only on $\x_{t-\hor(\varphi):t}\in\R^{d_x\times(\hor(\varphi)+1)}$. 

\begin{remark}[Why past-time STL]
Restricting to past-time formulas is natural for online monitoring: evaluation at time $t$ depends only on the finite history $\x_{t-\hor(\varphi):t}$, requiring no prediction of future states. Moreover, the $\wedge/\vee$-closed ptSTL fragment has a monotone, 1-Lipschitz algebraic structure used to enable tight compositional conformal bounds (Section~\ref{sec:exact}). 
\end{remark}

\subsection{Induced Specification Fragments}
We now define the class of specifications addressed in this paper. The key idea is to fix a finite dictionary of \emph{temporal atoms}, namely base ptSTL formulas that serve as irreducible generators, and close it under conjunction and disjunction (see Fig.~\ref{fig:fragment}).
The result is a fragment of ptSTL with a rich algebraic structure that admits tight compositional conformal certification via the semantic basis introduced in Section~\ref{sec:exact}.

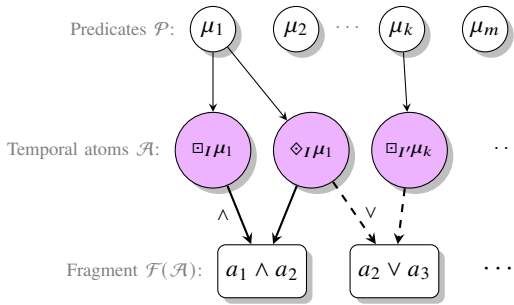
\begin{figure}[b]
\centering
\begin{tikzpicture}[
    >=stealth,
    ap/.style={draw, fill= white, circle, minimum size=6mm, inner sep=1pt, font=\small, drop shadow},
    atom/.style={draw, fill=accent!30!white, circle, minimum size=10mm, inner sep=1pt, font=\small, drop shadow},
    mol/.style={draw, fill=white, rounded corners=3pt, minimum height=7mm, inner sep=3pt, font=\small, drop shadow},
    lbl/.style={font=\scriptsize, text=gray}
]
\node[ap] (p1) at (0,0)    {$\mu_1$};
\node[ap] (p2) at (1.1,0)  {$\mu_2$};
\node[ap] (pk) at (2.5,0)  {$\mu_k$};
\node[ap] (pm) at (3.6,0)  {$\mu_m$};
\node[lbl] at (1.8, 0) {$\cdots$};
\node[lbl, left=2pt of p1] {Predicates $\PredLib$:};

\node[atom] (a1) at (0,   -1.6) {\scriptsize$\pAlways_I \mu_1$};
\node[atom] (a2) at (1.3, -1.6) {\scriptsize$\pEventually_I \mu_1$};
\node[atom] (a3) at (2.6, -1.6) {\scriptsize$\pAlways_{I'}\!\mu_k$};
\node[] (a4) at (3.9, -1.6) {\scriptsize$\cdots$};
\node[lbl, left=2pt of a1] {Temporal atoms $\Aset$:};

\draw[->] (p1) -- (a1);
\draw[->] (p1) -- (a2);
\draw[->] (pk) -- (a3);

\node[mol] (f1) at (0.65, -3.2) {$a_1\wedge a_2$};
\node[mol] (f2) at (2.4,  -3.2) {$a_2\vee a_3$};
\node[] (f3) at (3.8,  -3.2) {$\cdots$};
\node[lbl, left=2pt of f1] {Fragment $\SpecFam(\Aset)$:};

\draw[->, thick]        (a1) -- (f1) node[midway, left,  lbl, text=black] {$\wedge$};
\draw[->, thick]        (a2) -- (f1);
\draw[->, thick, dashed](a2) -- (f2) node[midway, right, lbl, text=black] {$\vee$};
\draw[->, thick, dashed](a3) -- (f2);

\end{tikzpicture}
\caption{Fragment structure. \emph{Top}: predicates $\mu_k\in\PredLib$.
\emph{Middle}: temporal atoms $a_q\in\Aset$, e.g., each applying a single past-time operator ($\pAlways_I$ or $\pEventually_I$) to one predicate.
\emph{Bottom}: induced fragment $\SpecFam(\Aset)$, formed by closing $\Aset$ under conjunction (solid, $\wedge$) and disjunction (dashed, $\vee$). 
}
\label{fig:fragment}
\end{figure}

\begin{definition}[Atomic dictionary and induced fragment]\label{def:fragment}
A finite set $\Aset=\{a_1,\dots,a_r\}$ of ptSTL formulas is an \emph{atomic dictionary}. The \emph{induced fragment} $\SpecFam(\Aset)$ is the smallest set containing $\Aset$ and closed under conjunction and disjunction:
\[
    \varphi\in\SpecFam(\Aset) \;\Longleftrightarrow\; \varphi ::= a_q \mid \varphi_1\wedge\varphi_2 \mid \varphi_1\vee\varphi_2,\quad a_q\in\Aset,
\]
with $\varphi_1,\varphi_2\in\SpecFam(\Aset)$.
The \emph{maximum horizon} of the fragment $\SpecFam(\Aset)$ is $\Kmax:=\max_{q}\,\hor(a_q)$.
\end{definition}

The elements of $\Aset$ are ``atomic'' in the sense that they are irreducible within $\SpecFam(\Aset)$: no formula in the fragment can be decomposed further below these atoms. Crucially, the choice of $\Aset$ is a design decision that determines the expressiveness of the fragment. Larger or richer dictionaries admit more complex specifications but require predicting a higher-dimensional semantic basis; see Section~\ref{sec:exact}.

\begin{example}\label{ex:depth_one_dictionary}
    In the experiments (Section~\ref{sec:experiments}), we use the depth-1 atomic dictionary
    \begin{equation}
        \Aset = \big\{\pAlways_I\,\mu_k,\;\pEventually_I\,\mu_k \mid \mu_k\in\PredLib,\;I\in\Iset\big\},\label{eq:depth_one_dictionary}
    \end{equation}
    where $\Iset$ is a finite set of time intervals, yielding $r=2m|\Iset|$ atoms. Each atom applies a single temporal operator to one predicate; the induced fragment then allows arbitrary $\wedge/\vee$ combinations of these temporal queries. The two-level structure is illustrated in Fig.~\ref{fig:fragment}.
\end{example}

In the following, we may suppress the dependence of $\SpecFam$ on $\Aset$ when the dictionary is clear from context, and write $\SpecFam$ to denote the target fragment for brevity.

\subsection{Problem Statement}\label{subsec:prob_form}

We consider vision-based monitors that, at each given time $t$, have access only to the observation history $\{o_\tau\}_{\tau=0}^t$ and not the true state history $\{\x_t\}_{\tau=0}^t$. In particular, we focus on monitors that operate on a learned latent representation of the observation history ($\basis_t$), which is a common approach in practice for vision-based systems; recall Fig.~\ref{fig:overview}.

To this end, an encoder $\enc_\theta\colon\O^{H} \to \R^{d_\basis}$ maps a sliding history of $H>0$ observations\footnote{Note that $H$ and $\Kmax$ are independent parameters: $\Kmax$ measures the length of the \emph{state} history required to evaluate formulas in the fragment $\SpecFam$, while $H$ measures the length of the \emph{observation} history fed into the encoder.} to a latent representation $\basis_t = \enc_\theta(o_{t-H+1:t}) \in \R^{d_\basis}$. At runtime, the monitor has access only to $\basis_t$; the physical state $x_t$ is never directly observed.

Let a dataset $\mathcal{D}=\{(o^i_{0:T_i},x^i_{0:T_i})\}_{i=1}^N$ with $N$ episodes of respective length $T_i>0$ drawn from the system in \eqref{eq:system} be split into training, calibration, and test episodes, and fix a target fragment $\SpecFam$ of temporal logic formulas (with maximum horizon $\Kmax$) and confidence level $1-\alpha\in[0,1]$.


\begin{problem}[Reusable Certified Online Monitoring]
      Construct a monitor that, at each valid time $t\ge \Kmax$, uses only the observation history $\{o_\tau\}_{\tau\le t}$ to output, for any queried formula $\varphi\in\SpecFam$, a
  certified lower bound $\underline\rho_t^\varphi\in\R$ such that:
      \begin{enumerate}[label=(\roman*),leftmargin=1.5em]
          \item \emph{Validity:} $\Prob(\underline\rho_t^\varphi \le \rho(\varphi,\x,t))\ge 1-\alpha$;
          \item \emph{Reusability:} a single trained encoder and a single calibration pass support all $\varphi\in\SpecFam$, with no per-formula retraining.
      \end{enumerate}
      \label{prob:main}
  \end{problem}

  We consider two instantiations of the validity guarantee, differing in what the probability in Problem~\ref{prob:main}(i) is taken over.
  \emph{Level-1} (episodewise): $\Prob$ is jointly over the $N$ calibration episodes and the test episode, and the bound holds simultaneously for all valid times $t$.
  \emph{Level-2} (random-time): $\Prob$ additionally includes a uniformly sampled evaluation time $\tau$ within the test episode, and the bound holds at $\tau$.

\section{SUFFICIENT STATISTICS FOR REUSABLE VISION-BASED MONITORING}
\label{sec:exact}


The central question of reusable monitoring is: \emph{what must the latent representation $\basis_t$ encode so that every formula in the target fragment $\SpecFam$ can be decoded from it by a monotone, 1-Lipschitz function, for any possible signal $\x$?} We seek the smallest such representation, thereby identifying the minimum prediction target needed to support the entire fragment. We restrict the decoder class to monotone, 1-Lipschitz functions, which are the natural choice for ptSTL because its robustness semantics are built from $\min$, $\max$, and coordinate projections, and because this is precisely the class that enables tight conformal certification (Section~\ref{sec:conformal}).

We present two choices of $\basis_t$ with complementary properties. (1) The \emph{predicate-history} basis $\basis^\PredLib_t$ is a \emph{fragment-agnostic} statistic: it supports every bounded-horizon ptSTL formula over $\PredLib$ without any prior knowledge of the target fragment. (2) The \emph{semantic basis} $\basis^\Aset_t$ is a \emph{fragment-specific} statistic: given a chosen fragment $\SpecFam(\Aset)$, it is the minimum representation from which every formula in the fragment admits a monotone, 1-Lipschitz decoder, uniformly over all signals.

\subsection{Predicate-History Basis}

The \emph{predicate history} collects the robustness values of all atomic predicates $\PredLib$ over the full fragment horizon $\Kmax$:
\begin{equation}
    \basis^\PredLib_t
    :=
    \big(\mu_k(\x,t-j)\big)_{k=1,\dots,m,\; j=0,\dots,\Kmax}
    \in \R^{m(\Kmax+1)}.\label{eq:predicate_history_basis}
\end{equation}
It is fragment-agnostic in the following sense.

\begin{proposition}[Predicate history factorization]\label{prop:universal}
    For any ptSTL formula $\varphi$ in PNF with predicates $\PredLib$ and $\hor(\varphi)\le\Kmax$, and any $t\ge \Kmax$, there exists a monotone, 1-Lipschitz (under $\|\cdot\|_\infty$) decoder $\dec_\varphi\colon\R^{m(\Kmax+1)}\to\R$ such that $$\rho(\varphi,\x,t)=\dec_\varphi(\basis^\PredLib_t).$$
\end{proposition}

\begin{proof}
Every predicate robustness value $\mu_k(\x,t-j)$ appearing in the evaluation of $\varphi$ is a coordinate of $\basis^\PredLib_t$. The decoder $\dec_\varphi$ is constructed by structural induction, composing coordinate projections, $\min$, and $\max$ according to the parse tree of $\varphi$. Each of these operations is monotone and 1-Lipschitz under $\|\cdot\|_\infty$, implying the result.
\end{proof}

The predicate history is the natural fragment-agnostic baseline: it can be predicted once and then decoded to any formula at query time, with no knowledge of the target fragment (within ptSTL over $\PredLib$) required at training or calibration time. Its dimension $m(\Kmax+1)$ grows linearly with the number of predicates and the fragment horizon, making it the most expensive representation we consider.

This motivates asking whether a smaller representation suffices when the target fragment is known. The following subsection answers this precisely.

\subsection{Semantic Basis}

Suppose a target fragment $\SpecFam(\Aset)$ has been fixed (Definition~\ref{def:fragment}). Rather than retaining the full predicate history \eqref{eq:predicate_history_basis}, we ask whether a smaller statistic suffices to evaluate every formula in $\SpecFam(\Aset)$. The answer is yes: it is enough to retain the robustness values of the atoms $a \in \Aset$. We call this basis the \emph{semantic basis} $\basis^\Aset_t$ which is the minimum in an information-theoretic sense (Definition~\ref{def:info_order}).

\begin{definition}[Semantic basis]\label{def:semantic_basis}
For an atomic dictionary $\Aset = \{a_1,\dots,a_r\}$, the \emph{semantic basis} is
\begin{equation}
    \basis^\Aset_t := \big(\rho(a_q,\x,t)\big)_{q=1}^r \in \R^r.\label{eq:semantic_basis}
\end{equation}
\end{definition}

The semantic basis stores exactly one robustness value per atom in $\Aset$. In general, a basis $\basis_t$ \emph{supports} $\SpecFam(\Aset)$ if every $\varphi\in\SpecFam(\Aset)$ admits a monotone, 1-Lipschitz decoder $\dec_\varphi$ satisfying $\rho(\varphi,\x,t)=\dec_\varphi(\basis_t)$ for all signals $\x$ and all times $t\geq \Kmax$. Uniformity over signals is crucial: without it, the downstream conformal guarantees would not transfer beyond the calibration distribution. To state the minimality claim precisely, we compare bases by their information content.

\begin{definition}[Information order]\label{def:info_order}     
  For two deterministic statistics $\basis_t^{(1)}=T_1(\basis^\PredLib_t)$ and $\basis_t^{(2)}=T_2(\basis^\PredLib_t)$, where $T_1,T_2$ are arbitrary deterministic maps, write $\basis_t^{(1)} \preceq              
  \basis_t^{(2)}$ if there exists a deterministic map $h$ such that $\basis_t^{(1)} = h(\basis_t^{(2)})$ for all signals $\x$ and all valid times $t\ge\Kmax$. Then $\basis_t^{(2)}$ is \emph{at least as informative
   as} $\basis_t^{(1)}$.  
\end{definition}

The semantic basis $\basis^\Aset_t$ is the minimum of this order among all statistics that support $\SpecFam(\Aset)$.

\begin{theorem}[Minimality of the semantic basis]
\label{thm:minimal_basis}
Let $\Aset=\{a_1,\dots,a_r\}$ and let $\SpecFam(\Aset)$ be its $\wedge/\vee$-closure.
\begin{enumerate}[label=(\roman*),leftmargin=1.5em]
\item For every $\varphi\in\SpecFam(\Aset)$, there exists a monotone, 1-Lipschitz (under $\|\cdot\|_\infty$) decoder $\dec_\varphi\colon\R^r\to\R$ such that $\rho(\varphi,\x,t)=\dec_\varphi(\basis^\Aset_t)$.
\item $\basis^\Aset_t$ is the minimum statistic: for every $\basis_t$ that supports $\SpecFam(\Aset)$, we have $\basis^\Aset_t \preceq \basis_t$.
\end{enumerate}
\end{theorem}

\begin{proof}
For (i), define $\dec_\varphi$ recursively: $\dec_{a_q}(b)=b_q$, $\dec_{\varphi_1\wedge\varphi_2}=\min\{\dec_{\varphi_1},\allowbreak\dec_{\varphi_2}\}$, $\dec_{\varphi_1\vee\varphi_2}=\max\{\dec_{\varphi_1},\allowbreak\dec_{\varphi_2}\}$. Monotonicity and 1-Lipschitz continuity follow because $\min$, $\max$, and coordinate projections have these properties under $\|\cdot\|_\infty$. For (ii), by (i) $\basis^\Aset_t$ supports $\SpecFam(\Aset)$. For minimality, each atom $a_q$ belongs to $\SpecFam(\Aset)$, so any $\basis_t$ that supports $\SpecFam(\Aset)$ admits a decoder $\overline{\dec}_q$ with $\rho(a_q,\x,t)=\overline{\dec}_q(\basis_t)$. Stacking gives $\basis^\Aset_t=\overline{\dec}(\basis_t)$, so $\basis^\Aset_t\preceq \basis_t$.
\end{proof}

Theorem~\ref{thm:minimal_basis} establishes that, for a fixed atomic dictionary $\Aset$ and decoder class restricted to monotone $1$-Lipschitz maps under $\|\cdot\|_\infty$, the semantic basis is the smallest prediction target that supports the entire fragment $\SpecFam(\Aset)$.

The predicate-history basis \eqref{eq:predicate_history_basis} recovers the semantic basis \eqref{eq:semantic_basis} in the limit; e.g., consider $\Aset$ in \eqref{eq:depth_one_dictionary} with degenerate intervals $\Iset=[j,j]$ for all $j\in\{0,\dots,\Kmax\}$, then $\basis^\Aset_t = \basis^\PredLib_t$, i.e., the two coincide and no compression is possible.

\section{CONFORMAL CERTIFICATION/CALIBRATION}
\label{sec:conformal}

The previous sections assume exact knowledge of $\basis_{t}\in\{\basis^\PredLib_{t},\basis^\Aset_{t}\}$. In practice, the encoder predicts an estimate $\widehat \basis_{t}$ from images (Fig.~\ref{fig:overview}), so prediction errors propagate into the decoded robustness predictions. We use CP to turn error bounds on the basis coordinates $\basis_{t,\ell}$ into valid lower bounds on the robustness of any formula in the fragments $\SpecFam(\PredLib)$ and $\SpecFam(\Aset)$, respectively. The key is that the monotone, 1-Lipschitz decoder structure allows us to certify all formulas simultaneously from a single set of conformal bounds on the basis elements $\basis_{t,\ell}$.

To this end, for each basis coordinate $\ell$, define the one-sided overestimation error $e_{t,\ell} := \max(0, \widehat \basis_{t,\ell} - \basis_{t,\ell})$ and let $\sigma_\ell > 0$ be a coordinatewise scaling factor. Based on this, we define the \emph{fragment-wide score} 
\begin{equation}
    s^{\SpecFam}(t) := \max_\ell\, \frac{e_{t,\ell}}{\sigma_\ell}.\label{eq:fragment_wide_score}
\end{equation}
For a formula $\varphi$, the \emph{active-support score} restricts to its basis coordinates: 
\begin{equation}
    s^\varphi(t) := \max_{\ell \in \supp(\varphi)} \frac{e_{t,\ell}}{\sigma_\ell} \le s^{\SpecFam}(t),\label{eq:active_support_score}
\end{equation}
where $\supp(\varphi)
$ denote the set of atom indices on which $\dec_\varphi$ depends. Errors in atoms outside $\supp(\varphi)$ do not affect the decoded robustness.
This can also be seen in Fig.~\ref{fig:fragment}, where each formula in the fragment $\SpecFam$ is associated with a subset of the basis coordinates through its support, and the score for each formula is the maximum normalized error over its active coordinates.

\begin{proposition}[Active-support error bound]
\label{prop:active}
For any estimate $\widehat \basis_t \in \R^r$ of basis $\basis_t$ and any $\varphi\in\SpecFam$ in the associated fragment, we have
\[
\big|\, \rho(\varphi,\x,t)-\dec_\varphi(\widehat \basis_t) \,\big|
\le
\max_{i\in\supp(\varphi)} |\basis_{t,i}-\widehat \basis_{t,i}|.
\]
\end{proposition}

\begin{proof}
$\dec_\varphi$ is 1-Lipschitz under $\|\cdot\|_\infty$ and depends only on coordinates in $\supp(\varphi)$.
\end{proof}

\begin{example}
In our experiments, the atomic dictionary is
$\Aset$ defined in \eqref{eq:depth_one_dictionary}
with $m=7$ predicates, $\Iset=\{[0,1],[0,2],[0,4],[0,8],[0,16]\}$, and $\Kmax=16$. The resulting semantic basis has $r = 2m|\Iset| = 70$ coordinates; the predicate history has $m(\Kmax+1) = 119$. This gives a $41\%$ reduction in representation size for the same target fragment.

\begin{figure}[th]
    \centering    \includegraphics[width=0.8\linewidth]{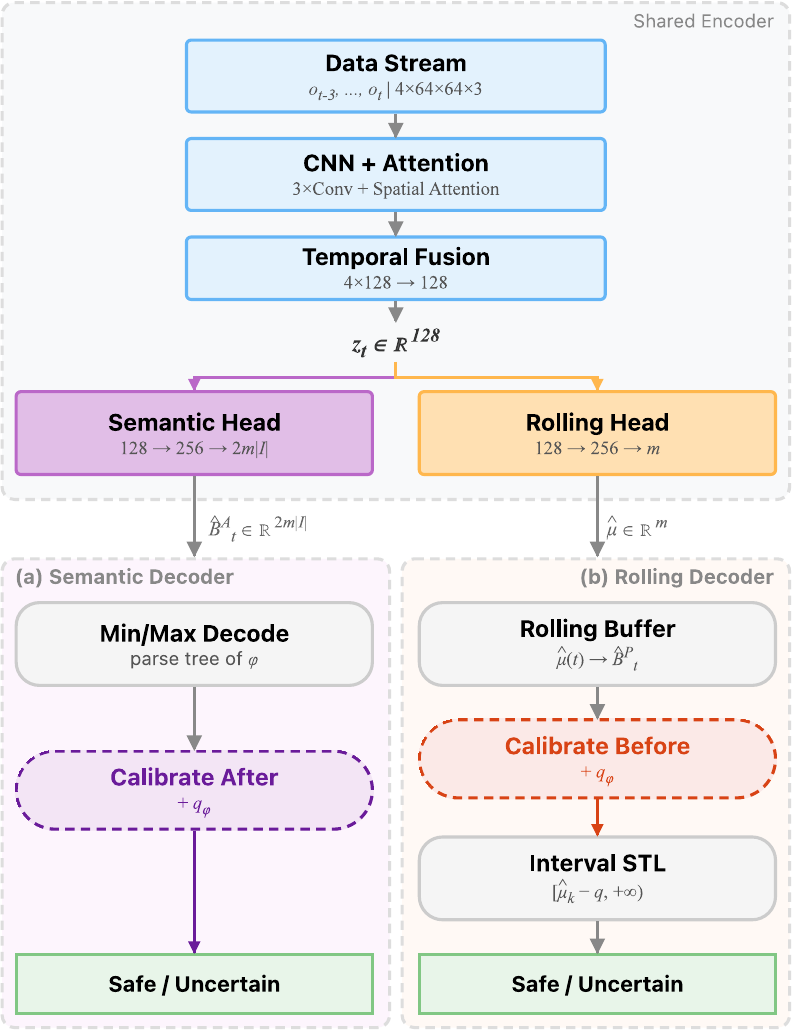}
    \caption{Monitor architectures for both benchmarks. Both share a CNN+attention encoder and temporal fusion module. \textbf{(a)}~The rolling monitor predicts $m$ current-timestep predicates, accumulates them in a streaming buffer,
and applies the conformal radius $q_\varphi$ \emph{before} temporal composition (interval STL). \textbf{(b)}~The semantic-basis monitor predicts $2m|\Iset|$ temporal atoms directly; $q_\varphi$ is applied 
\emph{after} composition via the formula's parse tree.}
    \label{fig:arch}
    \vspace{-10pt}
\end{figure}

As a concrete instance, the safety specification $\varphi_{\mathrm{safe}} = \pAlways_{[0,K]}\mu_{\mathrm{clear}}$ has decoder $\dec_{\varphi_{\mathrm{safe}}}(\basis^\Aset_t) = \rho(\pAlways_{[0,K]}\mu_{\mathrm{clear}}, \x, t)$---a single coordinate of $\basis^\Aset_t$. The reach-avoid specification $\varphi_{\mathrm{ra}} = \pEventually_{[0,K_g]}\mu_{\mathrm{goal}} \wedge \pAlways_{[0,K_c]}\mu_{\mathrm{clear}}$ has decoder $\dec_{\varphi_{\mathrm{ra}}}(\basis^\Aset_t) = \min\bigl\{\rho(\pEventually_{[0,K_g]}\mu_{\mathrm{goal}},\x,t),\,\rho(\pAlways_{[0,K_c]}\mu_{\mathrm{clear}},\x,t)\bigr\}$---a $\min$ of two basis coordinates.
\end{example}

\paragraph{Quantiles:}
Given calibration scores $S_1,\dots,S_n$, let
$$\qhat_{1-\alpha}(S_{1:n}) := S_{(\min\{n,\lceil (n+1)(1-\alpha)\rceil\})}$$
denote the split-conformal quantile. With $\widetilde C := \qhat_{1-\alpha}$ of the fragment-wide scores, the runtime lower bound on coordinate $\ell$ is $\underline \basis_{t,\ell} := \widehat \basis_{t,\ell} - \widetilde C\,\sigma_\ell$.

\begin{lemma}[Shared conformal bound]
\label{lem:shared_event}
If $\underline \basis_t \le \basis_t$ coordinatewise, then $\dec_\varphi(\underline \basis_t) \le \rho(\varphi,x,t)$ for all $\varphi \in \SpecFam$, by monotonicity of $\dec_\varphi$.
\end{lemma}

Lemma~\ref{lem:shared_event} is the key to reusability: conformal bounds on the basis elements $\basis_{t,\ell}$ are sufficient to certify every formula in the fragment $\SpecFam$ simultaneously, without a union bound.

\medskip

\begin{figure*}[t]
    \centering
    \includegraphics[width=13cm]{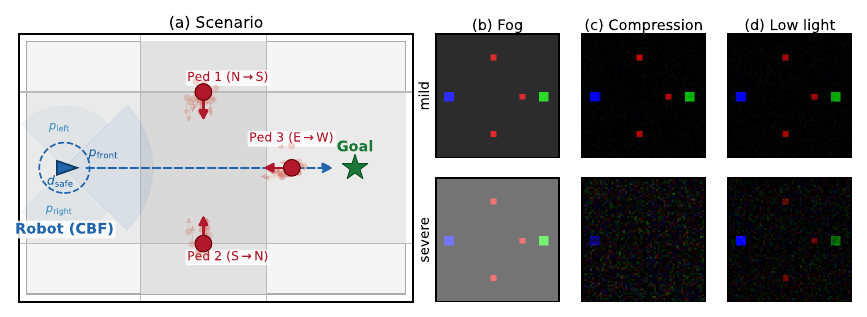}
    \vspace{-10pt}
    \caption{(a)~Crossroad scenario: robot (blue) navigates toward the goal (green) via a CBF controller while three pedestrians converge. Cones show predicate sectors ($45^\circ$ half-angle); dashed circle: $d_{\mathrm{safe}}{=}1.0$\,m.
    (b--d)~$64{\times}64$ observations under image nuisances (fog, compression, noise).}
    \label{fig:scenario}
    \vspace{-10pt}
\end{figure*}

\noindent\textbf{Temporal aggregation:}
The choice of score determines the strength of the guarantee. \emph{Level-1} uses the episode-wise maximum $S^{(i)} := \max_t s^{\SpecFam}(t)$, yielding a bound valid uniformly over all valid times $t\ge\Kmax$ and all $\varphi\in\SpecFam$ within a test episode. \emph{Level-2} samples one time $\tau_i\sim\mathrm{Unif}\{\Kmax,\dots,T_i\}$ per episode and sets $S^{(i)} := s^{\SpecFam}(\tau_i)$, giving a random-time guarantee at lower conservatism. We evaluate both levels experimentally.

\begin{theorem}[Simultaneous validity]
\label{thm:validity}
Under exchangeable episodes, with $\underline \basis_{t,\ell}:=\widehat \basis_{t,\ell}-\widetilde C\,\sigma_\ell$ and $\widetilde C=\qhat_{1-\alpha}(S_{1:n})$:
\begin{enumerate}[label=(\roman*),leftmargin=1.5em,topsep=2pt,itemsep=0pt]
    \item \emph{(Level-2)} $\Prob(\forall\varphi\in\SpecFam\colon\dec_\varphi(\underline \basis_\tau)\le\rho(\varphi,x,\tau))\ge 1{-}\alpha$.
    \item \emph{(Level-1)} $\Prob(\forall t\ge\Kmax,\,\forall\varphi\in\SpecFam\colon\dec_\varphi(\underline \basis_t)\le\rho(\varphi,x,t))\ge 1{-}\alpha$.
\end{enumerate}
\end{theorem}
\begin{proof}
Split conformal calibration on the exchangeable scores $\{S^{(i)}\}$ yields $\Prob(S^{(\mathrm{test})}\le\widetilde C)\ge 1{-}\alpha$, hence $\underline \basis\le \basis$ coordinatewise at the relevant time(s). Apply Lemma~\ref{lem:shared_event}.
\end{proof}

Since $s^\varphi(t)\le s^{\SpecFam}(t)$, restricting to the active support of a queried formula always yields a tighter or equal bound, at the cost of certifying only that formula rather than the whole fragment. Thus, $s^\varphi$ requires recalibrating when the query formula changes.

We write $q_\varphi := \qhat_{1-\alpha}(s^\varphi_{1:n})$ for the formula-specific conformal radius obtained from active-support scoring.

\section{MONITOR ARCHITECTURES}
\label{sec:arch}

All monitor variants share the same perception backbone (a CNN encoder mapping an observation window to a $128$-dimensional latent vector) and differ only in the prediction target and where conformal calibration is applied relative to temporal composition (see Fig.~\ref{fig:arch}). The rolling monitor predicts $m$ values per step (calibrated before composition; supports full ptSTL). The semantic-basis monitor predicts $2m|\Iset|$ values (calibrated after composition; supports $\SpecFam(\Aset)$).

\smallskip

\paragraph{Semantic-Basis Prediction:}
For the fragment induced by the atomic dictionary $\Aset$, the monitor predicts the semantic basis $\widehat\basis^\Aset_t \in \R^{2m|\Iset|}$.
By Theorem~\ref{thm:minimal_basis}, this is the minimum sufficient statistic for reusable monitoring over $\SpecFam(\Aset)$. Any queried formula is decoded by a deterministic $\min$/$\max$ tree derived from its parse tree---no per-formula training is required.

\smallskip

\paragraph{Rolling Prediction:}
The rolling monitor predicts only the current predicate vector $\hat\mu(t) \in \R^m$ (estimated from the observation window $o_{t-H+1:t}$) and accumulates predictions in a streaming buffer that reconstructs the predicate window $\widehat\basis^\PredLib_t$ online. Formula evaluation then applies the same interval-arithmetic decoder as any window-based monitor. This produces a larger representation than the semantic basis ($m(K_{\max}{+}1)$ vs. $2m|\Iset|$ entries for $\SpecFam(\Aset)$), but is easier to learn: the head solves a per-timestep regression ($m$ outputs) rather than predicting temporal aggregates ($2m|\Iset|$ outputs). Since the predicate window $\basis^\PredLib_t$ supports the full bounded-horizon ptSTL fragment $\SpecFam(\PredLib)$, the rolling monitor can certify any formula in $\SpecFam(\PredLib) \supseteq \SpecFam(\Aset)$. For the target fragment $\SpecFam(\Aset)$, this representation is sufficient but not minimal; the semantic basis provides a tighter interface.

\smallskip

\paragraph{Pre- vs.\ Post-Composition Calibration:}
A key distinction is whether conformal calibration is applied \emph{before} or \emph{after} temporal composition. The rolling monitor calibrates before: the conformal radius $q_\varphi$ is computed on raw per-timestep prediction errors, then propagated through the temporal $\min$/$\max$ operators of the STL formula. As the horizon grows, the score must protect against the worst error across more temporal lags, so $q_\varphi$ increases with $|\suppU(\varphi)|$. The semantic-basis monitor calibrates after: it predicts temporal aggregates directly, so $q_\varphi$ is computed on the aggregated output. Post-composition calibration avoids the horizon penalty, making $q_\varphi$ nearly insensitive to temporal depth, but at the cost of a harder prediction problem. Both architectures are encoder-agnostic: the prediction heads and conformal calibration depend only on the latent dimension, not the encoder architecture. A pretrained vision backbone (e.g., a ViT) could replace the CNN with only the head retrained.

\section{EXPERIMENTS}\label{sec:experiments}

We demonstrate that the optimal architecture depends on both the domain and the calibration level. On simulated data, a horizon-dependent crossover occurs: rolling wins at short horizons, semantic at long. On real-world driving data (Section~\ref{subsect:womd}), semantic dominates at all horizons under Level-2 calibration. Under the stronger Level-1 guarantee, rolling recovers the advantage on both benchmarks. Both architectures decisively outperform a Bonferroni-corrected observer baseline (Section~\ref{sec:observer_baseline}) on every formula tested.

\begin{figure*}[t]
    \centering
    \includegraphics[width=0.85\textwidth]{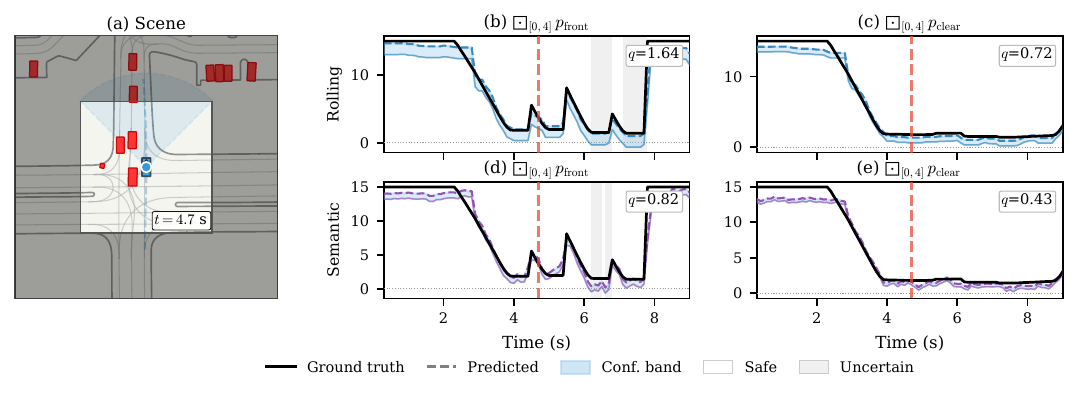}
    \vspace{-15pt}
\caption{%
Rolling and semantic monitors on a WOMD scenario ($88$ steps, $8.8$\,s).
(a)~Bird's-eye view of the ego vehicle (blue) and surrounding agents (red); inset: $30$\,m monitoring viewport.
(b--c)~Rolling monitor for $\pAlways_{[0,4]}p_{\mathrm{front}}$ and $\pAlways_{[0,4]}p_{\mathrm{clear}}$.
(d--e)~Semantic monitor for the same specifications. Ground truth: black; prediction: dashed; conformal band: blue; white $=$ \safe{}, gray $=$ \uncertain{}. Animated version: \href{https://youtu.be/B4XxOAdq9HI}{video}.%
} 
    \label{fig:WOMD}
    \vspace{-10pt}
\end{figure*}

\subsection{Crossroad Scenario}

A CBF-controlled robot navigates a pedestrian crossroad~\cite{IROS2024} (Fig.~\ref{fig:scenario}). The monitor observes $64{\times}64$ overhead images and predicts $m{=}7$ safety predicates (clearance, directional clearances, front margin, goal reach, speed margin) with $\Kmax{=}16$ and intervals
\[
\Iset=\{[0,1],[0,2],[0,4],[0,8],[0,16]\}.
\]

The rolling monitor predicts $7$ values per step; the semantic monitor predicts $70$ basis atoms. Both share a CNN encoder ($128$-dim latent, $H{=}4$ frame history; architecture details in the appendix). The crossroad dataset has $5{,}000$ training, $1{,}000$ calibration, and $500$ test episodes. Ground-truth predicates are computed from full state; at deployment, only the calibration set requires state access. All results use one-sided scoring with $\alpha{=}0.10$. All conformal radii in TABLE~\ref{tab:full_results} use formula-specific active-support scoring ($q_\varphi$): calibration residuals are stored once, and $q_\varphi$ is recomputed at query time from the cached scores. Changing the queried formula requires no new data collection or model inference.

\paragraph{Conformal Tightness:}
  Fig.~\ref{fig:crossover_plot} shows the conformal radius $q_\varphi$ vs.\ horizon $K$ for $\varphi=\pAlways_{[0,K]}p_f$: semantic's radius remains roughly constant while rolling's inflates steadily, driven by the support-size penalty that post-composition calibration avoids.
  Rolling is initially tighter below $K{\approx}3$ due to decoder complexity being kept equal but having to predict fewer values.
  By $K{=}16$, rolling reaches $q_\varphi{=}2.25$ while semantic's remains at $q_\varphi{=}0.56$---a 4-times gap (TABLE~\ref{tab:full_results}).

\begin{figure}[htb]
    \centering
    \begin{tikzpicture}
    \begin{axis}[
    width=0.82\columnwidth, height=3.2cm,
    xlabel={Horizon $K$}, ylabel={$q_\varphi$},
    legend pos=north west, legend style={font=\scriptsize},
    xtick={1,2,4,8,16}, xmode=log, log basis x=2,
    ymode=log, ymin=0.15, ymax=3, log basis y=4, grid=major,
    ytick={0.2, 0.5, 1, 2}, yticklabels={0.2, 0.5, 1, 2},
    label style={font=\small}, tick label style={font=\small},
    ]
    \addplot[color=myblue, mark=square*, thick] coordinates {(1,.206)(2,.263)(4,.380)(8,.773)(16,2.25)};
    \draw[dashed, gray, thin] (axis cs:3,0.15) -- (axis cs:3,3)
        node[pos=0.85, right, font=\tiny, text=gray] {$K{\approx}3$};
    \addlegendentry{Rolling}
    \addplot[color=myred, mark=triangle*, thick] coordinates {(1,.332)(2,.290)(4,.388)(8,.553)(16,.562)};
    \addlegendentry{Semantic}
    \end{axis}
    \end{tikzpicture}
    \vspace{-8pt}
    \caption{Conformal radius $q_\varphi$ vs.\ horizon $K$ for $\pAlways_{[0,K]}p_f$ (crossroad, Level-2).}
    \vspace{-15pt}
    \label{fig:crossover_plot}
\end{figure}
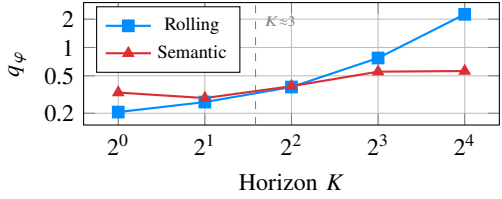
\paragraph{Observer Baseline:}\label{sec:observer_baseline}
We compare against an observer-style baseline~\cite{deshmukh2017robust,zhong2021extending,finkbeiner2022truly,lindemann2023conformal}. Using the same encoder, the baseline predicts per-predicate values, constructs symmetric conformal intervals, and propagates them through interval STL semantics. To ensure a valid $\alpha$-level guarantee, we apply a Bonferroni correction over the active predicate-lag support $\suppU(\varphi)$. The baseline provides the same coverage guarantees as Level-2, but with far looser radii due to the union bound (TABLE~\ref{tab:full_results}). Level-1 provides a stronger episodewise guarantee at the cost of larger quantiles.

\subsection{Real-World Validation: Waymo Open Motion Dataset}
\label{subsect:womd}

On the Waymo Open Motion Dataset (WOMD, v1.3.1)~\cite{Ettinger_2021_ICCV,Kan_2024_ICRA}, each scenario provides $8.8$\,s ($88$ timesteps at $10$\,Hz). We render $64{\times}64$ bird's eye view images and extract $m{=}7$ predicates (see TABLE~\ref{tab:full_results}), using the same encoder and $50{,}000$ training, $1{,}066$ calibration, $567$ test scenarios ($49{,}896$ timesteps). Calibration and test scenarios are drawn as disjoint random subsets of the \emph{validation\_interactive} split; exchangeability is assumed under i.i.d.\ sampling within the split.
To address distribution shift within the dataset (e.g., across geographic regions or weather conditions), robust conformal methods~\cite{zhao2024robust} can be applied.

\begin{table*}[t]
\centering
\small
\setlength{\tabcolsep}{3pt}
\setlength{\abovecaptionskip}{6pt}
\setlength{\belowcaptionskip}{4pt}
\caption{Results under Level-2 (random-time) and Level-1 (episodewise) calibration ($\alpha{=}0.10$).
CSR: certified safe rate. Prec: precision (fraction of safe certifications that are correct). FPR: false-positive rate. GT: ground-truth safe rate (``---'' when GT${=}100\%$).}
\vspace{-10pt}
\label{tab:full_results}
\resizebox{\textwidth}{!}{%
\begin{tabular}{@{}l c cccc cccc cccc cc cc @{}}
\toprule
& & & & & & \multicolumn{8}{c}{Level-2} & \multicolumn{4}{c}{Level-1$^*$} \\
\cmidrule(lr){7-14} \cmidrule(lr){15-18}
& & \multicolumn{4}{c}{Observer Baseline} & \multicolumn{4}{c}{Semantic} & \multicolumn{4}{c}{Rolling} & \multicolumn{2}{c}{Semantic} & \multicolumn{2}{c}{Rolling} \\
\cmidrule(lr){3-6} \cmidrule(lr){7-10} \cmidrule(lr){11-14} \cmidrule(lr){15-16} \cmidrule(lr){17-18}
Specification $\varphi$ & GT\%
& $q_\varphi$\,$\downarrow$ & CSR\,$\uparrow$ & Prec\,$\uparrow$ & FPR\,$\downarrow$
& $q_\varphi$\,$\downarrow$ & CSR\,$\uparrow$ & Prec\,$\uparrow$ & FPR\,$\downarrow$
& $q_\varphi$\,$\downarrow$ & CSR\,$\uparrow$ & Prec\,$\uparrow$ & FPR\,$\downarrow$
& $q_\varphi$\,$\downarrow$ & CSR\,$\uparrow$
& $q_\varphi$\,$\downarrow$ & CSR\,$\uparrow$ \\
\midrule
\rowcolor{gray!20}
\multicolumn{18}{@{}l}{\textbf{Crossroad}} \\
\multicolumn{18}{@{}l}{\emph{Horizon scaling ($p_f$)}} \\
\rowcolor{gray!10}
$\pAlways_{[0,1]}p_f$  & 100 & 1.36 & 86.7 & 100 & --- & .33 & 98.9 & 99.9 & --- & .21 & 98.6 & 99.9 & --- & 5.78 & 54.2 & 6.67 & 53.0 \\
$\pAlways_{[0,4]}p_f$  & 100 & 3.76 & 55.9 & 100 & --- & .39 & 96.7 & 99.9 & --- & .38 & 95.8 & 99.9 & --- & 5.82 & 49.8 & 6.67 & 49.1 \\
\rowcolor{gray!10}
$\pAlways_{[0,16]}p_f$ & 100 & 6.68 & 36.2 & 100 & --- & .56 & 87.4 & 99.7 & --- & 2.25 & 60.0 & 99.9 & --- & 5.61 & 36.5 & 6.67 & 36.2 \\
\addlinespace
\multicolumn{18}{@{}l}{\emph{Compound}} \\
\rowcolor{gray!10}
$\pAlways_{[0,4]}p_f{\wedge}\pAlways_{[0,4]}p_l$ & 98 & 5.62 & 12.5 & 99.9 & 0.1 & 1.00 & 76.1 & 99.8 & 0.1 & 1.12 & 72.4 & 99.9 & 0.1 & 6.46 & 8.1 & 7.47 & 5.6 \\
\addlinespace
\multicolumn{18}{@{}l}{\emph{Eventually}} \\
\rowcolor{gray!10}
$\pEventually_{[0,4]}p_f$  & 100 & 3.76 & 68.5 & 100 & --- & .28 & 99.7 & 100 & --- & .38 & 99.9 & 100 & --- & 5.89 & 57.6 & 6.67 & 56.3 \\
\midrule[0.08em]
\rowcolor{gray!20}
\multicolumn{18}{@{}l}{\textbf{WOMD}} \\
\multicolumn{18}{@{}l}{\emph{Horizon scaling ($p_f$)}} \\
\rowcolor{gray!10}
$\pAlways_{[0,1]}p_f$  & 100 & 2.01 & 90.2 & 100 & --- & .81 & 98.1 & 100 & --- & .91 & 98.1 & 100 & --- & 5.96 & 70.7 & 4.27 & 78.8 \\
$\pAlways_{[0,4]}p_f$  & 100 & 3.00 & 79.2 & 100 & --- & .82 & 96.9 & 100 & --- & 1.64 & 91.0 & 100 & --- & 6.36 & 64.1 & 4.27 & 75.4 \\
\rowcolor{gray!10}
$\pAlways_{[0,16]}p_f$ & 100 & 5.49 & 46.4 & 100 & --- & 1.91 & 85.5 & 100 & --- & 2.94 & 74.4 & 100 & --- & 8.78 & 40.7 & 4.27 & 65.5 \\
\addlinespace
\multicolumn{18}{@{}l}{\emph{Safety-critical predicates ($K{=}4$)}} \\
\rowcolor{gray!10}
$\pAlways_{[0,4]}p_s$  & 96 & 7.74 & 30.5 & 99.8 & 1.9 & 2.11 & 85.0 & 98.5 & 31.9 & 2.25 & 71.9 & 99.4 & 10.4 & 6.02 & 52.0 & 5.29 & 46.1 \\
$\pAlways_{[0,4]}p_\tau$ & 62 & 7.29 & 0.1 & 75.9 & 0.1 & 1.92 & 32.6 & 91.4 & 7.4 & 2.03 & 16.3 & 96.7 & 1.4 & 5.00 & 4.3 & 4.15 & 3.0 \\
\rowcolor{gray!10}
$\pAlways_{[0,4]}p_h$ & 40 & 13.22 & 0.9 & 99.6 & 0.0 & 3.88 & 21.2 & 94.3 & 2.0 & 4.93 & 14.1 & 98.7 & 0.3 & 11.85 & 3.6 & 11.00 & 3.8 \\
\addlinespace
\multicolumn{18}{@{}l}{\emph{Compound}} \\
\rowcolor{gray!10}
$\pAlways_{[0,4]}p_f{\wedge}\pAlways_{[0,4]}p_l{\wedge}\pAlways_{[0,4]}p_r$ & 100 & 5.80 & 42.8 & 100 & --- & 1.59 & 76.8 & 100 & --- & 2.29 & 68.4 & 100 & --- & 8.11 & 32.6 & 5.82 & 43.4 \\
$\pAlways_{[0,4]}p_f{\wedge}\pAlways_{[0,4]}p_\tau$ & 62 & 7.29 & 0.1 & 96.0 & 0.1 & 2.25 & 25.3 & 92.8 & 4.8 & 2.58 & 9.5 & 97.2 & 0.7 & 6.66 & 0.5 & 5.11 & 0.6 \\
\rowcolor{gray!10}
$\pAlways_{[0,4]}p_s{\wedge}\pAlways_{[0,4]}p_\tau{\wedge}\pAlways_{[0,4]}p_h$ & 34 & 13.22 & 0.9 & 100 & 0.0 & 4.48 & 2.6 & 94.9 & 0.2 & 5.31 & 0.2 & 100 & 0.0 & 11.94 & 0.0 & 11.02 & 0.0 \\
\addlinespace
\multicolumn{18}{@{}l}{\emph{Eventually}} \\
\rowcolor{gray!10}
$\pEventually_{[0,4]}p_f$  & 100 & 3.00 & 91.3 & 100 & --- & .81 & 99.2 & 100 & --- & 1.64 & 97.2 & 100 & --- & 4.99 & 81.7 & 4.27 & 84.7 \\
$\pEventually_{[0,4]}p_\tau$ & 72 & 7.29 & 0.6 & 99.3 & 0.0 & 2.01 & 42.2 & 93.5 & 9.8 & 2.03 & 33.1 & 94.9 & 6.0 & 5.03 & 6.3 & 4.15 & 8.7 \\
\bottomrule
\end{tabular}}\\[2pt]
{\footnotesize $^*$Level-1 omits Prec and FPR (Prec $100\%$, FPR $0$ for all specs).}
\vspace{-10pt}
\end{table*} 

TABLE~\ref{tab:full_results} reports both architectures on both benchmarks under Level-2 and Level-1 calibration. Fig.~\ref{fig:WOMD} shows a single WOMD scenario; \safe{} (white) and \uncertain{} (gray) regions are separated by the zero crossing of the conformal lower bound $\hat\mu - q_\varphi$.

\textit{1. Semantic is uniformly tighter.} At Level-2, semantic achieves a tighter conformal radius at every horizon on WOMD ($q_\varphi{=}0.81$ vs.\ $0.91$ at $K{=}1$; $q_\varphi{=}1.91$ vs.\ $2.94$ at $K{=}16$). Unlike in the crossroad experiment, rolling shows no initial advantage, likely because the learnability gap between the architectures is smaller on real-world data.

\textit{2. Soundness and conservatism.} Both monitors are empirically sound: empirical coverage stays above $1{-}\alpha{=}90\%$ on all specifications, confirming the coverage guarantee of Problem~\ref{prob:main}(i). The key distinction is conservatism: semantic certifies substantially more timesteps (e.g., $32.6\%$ vs.\ $16.3\%$ CSR on $\pAlways_{[0,4]}p_\tau$) because post-composition calibration yields a tighter $q_\varphi$.

\textit{3. Liveness and compound formulas.} Switching from $\pAlways$ to $\pEventually$ recovers substantial CSR: front clearance rises from $91\%$ to $97\%$ (rolling) and from $97\%$ to $99\%$ (semantic). The compound lane-change rule $\pAlways_{[0,4]}p_f{\wedge}\pAlways_{[0,4]}p_l{\wedge}\pAlways_{[0,4]}p_r$ certifies $77\%$ (semantic) and $68\%$ (rolling) of timesteps with zero false certifications.

\section{CONCLUSION}
\label{sec:conclusion}

We presented a framework for certified reusable monitoring from vision. The semantic-basis monitor predicts the minimal representation needed to decode every formula in the target fragment via a monotone, 1-Lipschitz decoder, enabling fragment-wide certification from a single conformal calibration pass. The rolling monitor trades this minimality for a simpler per-step learning problem by calibrating before temporal composition.
Both architectures outperform a Bonferroni-corrected observer baseline on every tested formula. Their ranking depends on the domain and calibration level: on crossroad, a crossover occurs near $K{\approx}3$, with rolling tighter at short horizons and semantic tighter at long horizons (TABLE~\ref{tab:full_results}). The framework is
  encoder-agnostic. The semantic basis is also amenable to specification mining~\cite{hoxha2018mining}. the monotone decoder structure allows the predicted atom values to directly reveal which specifications are satisfied without enumerating the fragment. Future work includes richer modalities and adaptive conformal methods~\cite{gibbs2021adaptive}.


\bibliographystyle{IEEEtran}
\bibliography{library}


\end{document}